\documentclass[12pt]{article}
\usepackage{fullpage}
\usepackage{hyperref}
\usepackage{amsthm, amsmath, amsfonts}
\usepackage{bbm}
\usepackage{graphicx}
\usepackage{pagecolor}
\usepackage{caption}
\usepackage{subcaption}
\usepackage{mathtools}
\usepackage{verbatim}
\usepackage{epstopdf}

\newtheorem{theor}{Theorem}
\newtheorem{corol}{Corollary}
\newtheorem{lemma}{Lemma}

\newcommand{\footremember}[2]{%
    \footnote{#2}
    \newcounter{#1}
    \setcounter{#1}{\value{footnote}}%
}

\newcommand{\Mod}[1]{\ (\mathrm{mod}\ #1)}

\begin{document}
\title{Quantified advantage of discontinuous weight selection in approximations with deep neural networks}
\author{Dmitry Yarotsky\footremember{skoltech}{Skolkovo Institute of Science and Technology, Skolkovo Innovation Center, Building 3, Moscow  143026
Russia}\footremember{iitp}{Institute for Information Transmission Problems, Bolshoy Karetny per. 19, build.1, Moscow 127051, Russia}\\
\texttt{d.yarotsky@skoltech.ru}
}
\date{}
\maketitle

\begin{abstract} We consider approximations of 1D Lipschitz functions by deep ReLU networks of a fixed width. We prove that without the assumption of continuous weight selection the uniform approximation error is lower than with this assumption at least by a factor logarithmic in the size of the network. 

\medskip
\noindent
{\bf Keywords: } Continuous nonlinear widths; Deep neural networks; ReLU activation
\end{abstract}

\section{Introduction}
In the study of parametrized nonlinear approximations, the framework of continuous nonlinear widths of DeVore \emph{et al.} \cite{devore1989optimal} is a general approach based on the assumption that approximation parameters depend continuously on the approximated object. Under this assumption, the framework provides lower error bounds for approximations in several important functional spaces, in particular Sobolev spaces. 

In the present paper we will only consider the simplest Sobolev space $W^{1,\infty}([0,1]),$ i.e. the space of Lipschitz functions on the segment $[0,1]$. The norm in this space can be defined by $\|f\|_{W^{1,\infty}([0,1])}=\max(\|f\|_\infty, \|f'\|_\infty),$ where $f'\in L^\infty([0,1])$ is the weak derivative of $f$, and $\|\cdot\|_\infty$ is the norm in $L^\infty([0,1])$. We will denote by $B$ the unit ball in $W^{1,\infty}([0,1])$. 
Then the following result contained in Theorem 4.2 of \cite{devore1989optimal} establishes an asymptotic lower bound on the accuracy of approximating functions from $B$ under assumption of continuous parameter selection:
\begin{theor}\label{th:1}
Let $\nu$ be a positive integer and $\eta:\mathbb R^\nu\to L^\infty([0,1])$ be any map between the space $\mathbb R^\nu$ and the space $L^\infty([0,1])$. Suppose that there is a continuous map $\psi:B\to \mathbb R^\nu$ such that $\|f-\eta(\psi(f))\|_\infty\le \epsilon$ for all $f\in B$. Then $\epsilon\ge \frac{c}{\nu}$ with some absolute constant $c$.
\end{theor}

The bound $\frac{c}{\nu}$ stated in the theorem  is attained, for example, by the usual piecewise-linear interpolation with uniformly distributed breakpoints. Namely, assuming $\nu\ge 2$, define $\psi(f)=(f(0), f(\frac{1}{\nu-1}),\ldots, f(1)))$ and define $\eta(y_1,\ldots,y_\nu)$ to be the continuous piecewise-linear function with the nodes $(\frac{n-1}{\nu-1}, y_n), n=1,\ldots,\nu.$ Then it is easy to see that $\|f-\eta(\psi(f))\|_\infty\le \frac{1}{\nu-1}$ for all $f\in B.$  

The key assumption of Theorem \ref{th:1} that  $\psi$ is continuous need not hold in applications, in particular for approximations with neural networks. The most common practical task in this case is to optimize the weights of the network so as to obtain the best approximation of a specific function, without any regard for other possible functions. Theoretically, Kainen \emph{et al.} \cite{kainen1999approximation} prove that already for networks with a single hidden layer the optimal weight selection is discontinuous in general. The goal of the present paper is to quantify the accuracy gain that discontinuous weight selection brings to a deep feed-forward neural network model in comparison to the baseline bound of Theorem \ref{th:1}.

\begin{figure}
\begin{center}
\includegraphics[clip,trim=15mm 5mm 5mm 5mm]{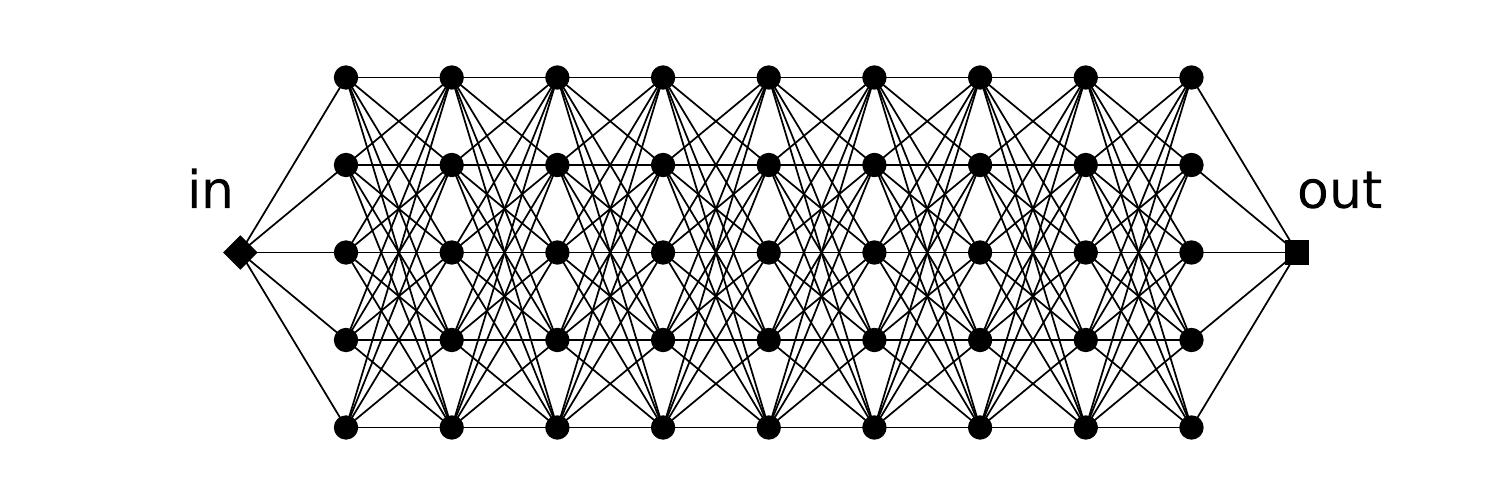}
\caption{The standard network architecture of depth $N=9$ and width $M=5$ (i.e., containing 9 fully-connected hidden layers, 5 neurons each).}\label{fig:stnet}
\end{center}
\end{figure}

Specifically, we consider a common deep architecture shown in Fig. \ref{fig:stnet} that we will refer to as \emph{standard}. The network consists of one input unit, one output unit, and $N$ fully-connected hidden layers, each including a constant number $M$ of units. We refer to $N$ as the \emph{depth} of the network and to $M$ as its \emph{width}. The layers are fully connected in the sense that each unit is connected with all units of the previous and the next layer. Only neighboring layers are connected. A hidden unit computes the map $$(x_1,\ldots,x_K)\mapsto\sigma\Big(\sum_{k=1}^K w_kx_k+h\Big),$$
where $x_1,\ldots,x_K$ are the inputs, $\sigma$ is a nonlinear activation function, and $(w_k)_{k=1}^K$ and $h$ are the weights associated with this computation unit. For the first hidden layer $K=1$, otherwise $K=M$. The network output unit is assumed to compute a linear map:
$$(x_1,\ldots,x_M)\mapsto\sum_{k=1}^M w_kx_k+h.$$
The total number of weights in the network then equals $$\nu_{M,N}=M^2(N-1)+M(N+2)+1.$$  

We will take the activation function $\sigma$ to be the ReLU function (Rectified Linear Unit), which is a popular choice in applications:
\begin{equation}\label{eq:relu}\sigma(x)=\max(0,x).
\end{equation}
With this activation function, the functions implemented by the whole network are continuous and piecewise linear, in particular they belong to $W^{1,\infty}([0,1])$. Let  $\eta_{M,N}:\mathbb R^{\nu_{M,N}}\to W^{1,\infty}([0,1])$ map the weights of the width-$M$ depth-$N$ network to the respective implemented function. 
Consider uniform approximation rates for the ball $B$, with or without the weight selection continuity assumption:
\begin{align*}
d_{\rm cont}(M,N)=&\inf_{\substack{\psi:B\to\mathbb R^{\nu_{M,N}}\\\psi\text{ continuous}}}\sup_{f\in B}\|f-\eta_{M,N}(\psi(f))\|_\infty\\
d_{\rm all}(M,N)=&\inf_{\psi:B\to\mathbb R^{\nu_{M,N}}}\sup_{f\in B}\|f-\eta_{M,N}(\psi(f))\|_\infty
\end{align*}
Theorem \ref{th:1} implies that $d_{\rm cont}(M,N)\ge \frac{c}{\nu_{M,N}}$.

We will be interested in the scenario where the width $M$ is fixed and the depth $N$ is varied. The main result of our paper is an upper bound for $d_{\rm all}(M,N)$ that we will prove in the case $M=5$ (but the result obviously also holds for larger $M$).   
\begin{theor}\label{th:2}
$$d_{\rm all}(5,N)\le \frac{c}{N\ln N},$$
with some absolute constant $c$.
\end{theor}
Comparing this result with Theorem \ref{th:1}, we see that without the assumption of continuous weight selection the approximation error is lower than with this assumption at least by a factor logarithmic in the size of the network:
\begin{corol}
$$\frac{d_{\rm all}(5,N)}{d_{\rm cont}(5,N)}\le \frac{c}{\ln N},$$
with some absolute constant $c$.
\end{corol}

The remainder of the paper is the proof of Theorem \ref{th:2}. The proof relies on a construction of adaptive network architectures with ``cached'' elements from \cite{yarotsky2016error_v3}. We use a somewhat different, streamlined version of this construction: whereas in \cite{yarotsky2016error_v3} it was first performed for a discontinuous activation function and then extended to ReLU by some approximation, in the present paper we do it directly, without invoking auxiliary activation functions.

\section{Proof of Theorem \ref{th:2}}

It is convenient to divide the proof into four steps. In the first step we construct for any function $f\in B$ an approximation $\widetilde f$ using cached functions. In the second step we expand the constructed approximation in terms of the ReLU function \eqref{eq:relu} and linear operations. In the third step we implement $\widetilde f$ by a highly parallel shallow network with an $f$-dependent (``adaptive'') architecture.  In the fourth step we show that this adaptive architecture can be embedded in the standard deep architecture of width 5.


\paragraph{Step 1: Cache-based approximation.}

We first explain the idea of the construction. We start with interpolating $f$ by a piece-wise linear function $\widetilde f_1$ with uniformly distributed nodes. After that, we create a ``cache'' of auxiliary functions that will be used for detalization of the approximation in the intervals between the nodes. The key point is that one cached function can be used in many intervals, which will eventually lead to a $\sim\log N$ saving in the error--complexity relation. The assignment of cached functions to the intervals will be $f$-dependent and encoded in the network architecture in Step 3.

The idea of reused subnetworks dates back to the proof of the $O(2^n/n)$ upper bound for the complexity of Boolean circuits implementing $n$-ary functions (\cite{shannon1949synthesis}).  

We start now the detailed exposition. Given $f\in B$, we will construct an approximation $\widetilde f$ to $f$ in the form
$$\widetilde f= \widetilde f_1+\widetilde f_2.$$
Here, $\widetilde f_1$ is the piecewise linear interpolation of $f$ with the breakpoints $\{\frac{t}{T}\}_{t=0}^T$:
$$\widetilde f_1\Big(\frac{t}{T}\Big)=f\Big(\frac{t}{T}\Big), \quad t=0,1,\ldots,T.$$
The value of $T$ will be chosen later.

Since $f$ is Lipschitz with constant 1, $\widetilde f_1$ is also Lipschitz with constant 1. We denote by $I_t$ the intervals between the breakpoints:
$$I_t=\Big[\frac{t}{T}, \frac{t+1}{T}\Big),\quad t=0,\ldots,T-1.$$
We will now construct $\widetilde f_2$ as an approximation to the difference \begin{equation}\label{eq:f2}f_2=f-\widetilde f_1.\end{equation} 
Note that $f_2$ vanishes at the endpoints of the intervals $I_t$:
\begin{equation}\label{eq:f2pr1} f_2\Big(\frac{t}{T}\Big)=0,\;\; t=0,\ldots,T,
\end{equation}
and $f_2$ is Lipschitz with constant 2: 
\begin{equation}\label{eq:f2pr2} |f_2(x_1)-f_2(x_2)|\le 2|x_1-x_2|,
\end{equation}
since $f$ and $\widetilde f_1$ are Lipschitz with constant 1.

To define $\widetilde f_2$, we first construct a set $\Gamma$ of cached functions. Let $m$ be a positive integer to be chosen later. Let $\Gamma$ be the set of piecewise linear functions $\gamma: [0,1]\to\mathbb R$ with the breakpoints $\{\frac{r}{m}\}_{r=0}^m$ and the properties 
$$\gamma(0)=\gamma(1)=0$$
and
$$\gamma\Big(\frac{r}{m}\Big)-\gamma\Big(\frac{r-1}{m}\Big)\in \Big\{-\frac{2}{m}, 0, \frac{2}{m}\Big\},\quad r =1, \ldots, m.$$
Note that the size $|\Gamma|$ of $\Gamma$ is not larger than $3^m$.

If  $g:[0,1]\to \mathbb R$ is any Lipschitz function with constant 2 and $g(0)=g(1)=0$, then $g$ can be approximated by some $\gamma\in \Gamma$ with error not larger than $\frac{2}{m}$: namely, take $\gamma(\frac{r}{m})=\frac{2}{m}\lfloor g(\frac{r}{m})/\frac{2}{m}\rfloor$ (here $\lfloor\cdot\rfloor$ is the floor function.)

Moreover, if $f_2$ is defined by \eqref{eq:f2}, then, using \eqref{eq:f2pr1}, \eqref{eq:f2pr2}, on each interval $I_t$ the function $f_2$ can be approximated with error not larger than $\frac{2}{Tm}$ by a properly rescaled function $\gamma\in\Gamma$. Namely, for each $t=0,\ldots,T-1$ we can define the function $g$ by $g(y)=Tf_2(\frac{t+y}{T})$. Then it is Lipschitz with constant 2 and $g(0)=g(1)=0$, so we can find $\gamma_t\in\Gamma$ such that
\begin{equation*}
\sup_{y\in [0,1)}\Big|Tf_2\Big(\frac{t+y}{T}\Big)-\gamma_t(y)\Big| \le \frac{2}{m}.
\end{equation*}
This can be equivalently written as 
\begin{equation*}
\sup_{x\in I_t}\Big|f_2(x)-\frac{1}{T}\gamma_t(Tx-t)\Big| \le \frac{2}{Tm}.
\end{equation*}
Note that the obtained assignment $t\mapsto \gamma_t$ is not injective, in general ($T$ will be larger than $|\Gamma|$).

We can then define $\widetilde f_2$ on the whole interval $[0,1)$ by
\begin{equation}\label{eq:f2i}
\widetilde f_2(x)=\frac{1}{T}\gamma_t(Tx-t),\quad x\in I_t, \quad t=0,\ldots,T-1.
\end{equation}
This $\widetilde f_2$ approximates $f_2$ with error $\frac{2}{Tm}$ on $[0,1)$: 
\begin{equation*}
\sup_{x\in[0,1)}|f_2(x)-\widetilde f_2(x)|\le \frac{2}{Tm},
\end{equation*}
and hence, by \eqref{eq:f2}, for the full approximation $\widetilde f=\widetilde f_1+\widetilde f_2$ we will also have
\begin{equation}\label{eq:f2tm}
\sup_{x\in[0,1)}|f(x)-\widetilde f(x)|\le \frac{2}{Tm}.
\end{equation}

\paragraph{Step 2: ReLU-expansion of $\widetilde f$.} We express now the constructed approximation $\widetilde f=\widetilde f_1+\widetilde f_2$ in terms of linear and ReLU operations. 

Let us first describe the expansion of $\widetilde f_1$. Since $\widetilde f_1$ is a continuous piecewise-linear interpolation of $f$ with the breakpoints $\{\frac{t}{T}\}_{t=0}^T$, we can represent it on the segment $[0,1]$ in terms of the ReLU activation function $\sigma$ as 
\begin{equation}\label{eq:f1exp}
\widetilde f_1(x)=\sum_{t=0}^{T-1}w_t\sigma\Big(x-\frac{t}{T}\Big)+h,
\end{equation}
with some weights $(w_t)_{t=0}^{T-1}$ and $h$.

Now we turn to $\widetilde f_2$, as given by \eqref{eq:f2i}. Consider the ``tooth'' function
\begin{equation}
\phi(x)=
\begin{cases}
0, & x\notin[-1,1],\\
1-|x|,& x\in[-1,1].
\end{cases}
\end{equation}
Note in particular that for any $t=0,1,\ldots,T-1$
\begin{equation}\label{eq:phivanish}
\phi(Tx-t)=0\quad \text{if } x\notin I_{t-1}\cup I_{t}.
\end{equation}
The function $\phi$ can be written as 
\begin{equation}\label{eq:phiexp}\phi(x)= \sum_{q=-1}^1\alpha_q\sigma(x-q),
\end{equation}
where $$\alpha_{-1}=\alpha_1=1, \alpha_0=-2.$$
Let us expand each $\gamma\in\Gamma$ over the basis of shifted ReLU functions:
\begin{equation}\label{eq:gammaexp}
\gamma(x) = \sum_{r=0}^{m-1}c_{\gamma,r}\sigma\Big(x-\frac{r}{m}\Big),\quad x\in[0,1],
\end{equation}
with some coefficients $c_{\gamma,r}$. There is no constant term because $\gamma(0)=0$. Since $\gamma(1)=0$, we also have
\begin{equation}\label{eq:cgr1}
\sum_{r=0}^{m-1}c_{\gamma,r}\sigma\Big(1-\frac{r}{m}\Big) = 0.
\end{equation}
Consider the functions $\theta_\gamma:\mathbb R^2\to\mathbb R$ defined by 
\begin{equation}\label{eq:ggammaexp}
\theta_\gamma(a,b)=\sum_{r=0}^{m-1}c_{\gamma,r}\sigma\Big(\frac{m-r}{m}a-\frac{r}{m}b\Big).
\end{equation}

\begin{lemma} 
\begin{align}\label{eq:ggamma11}
\theta_\gamma(a,0)&=0\quad\text{ for all } a\ge 0,\\
\label{eq:ggamma12}
\theta_\gamma(0,b)&=0\quad\text{ for all } b\ge 0,\\
\label{eq:ggamma2}
\theta_{\gamma}(\phi(Tx-t-1),\phi(Tx-t))&=\begin{cases} \gamma(Tx-t),& x\in I_t,\\
0, &x\notin I_{t},\end{cases}\end{align}
for all $x\in[0,1]$ and $t=0,1,\ldots,T-1$.
\end{lemma}
\begin{proof}
Property \eqref{eq:ggamma11} follows from \eqref{eq:cgr1} using positive homogeneity of $\sigma$.  Property \eqref{eq:ggamma12} follows since $\sigma(x)=0$ for $x\le 0$. 

To establish \eqref{eq:ggamma2}, consider the two cases:
\begin{enumerate}
\item If $x\notin I_t,$ then at least one of the arguments of $\theta_\gamma$ vanishes by \eqref{eq:phivanish}, and hence the l.h.s. vanishes by \eqref{eq:ggamma11}, \eqref{eq:ggamma12}. 
\item If $x\in I_t,$ then $\phi(Tx-t-1)=Tx-t$ and $\phi(Tx-t)=1+t-Tx$, so 
$$
\frac{m-r}{m}\phi(Tx-t-1)-\frac{r}{m}\phi(Tx-t)=Tx-t-\frac{r}{m}
$$
and hence, by \eqref{eq:ggammaexp} and \eqref{eq:gammaexp},
$$
\theta_{\gamma}(\phi(Tx-t-1),\phi(Tx-t))=\sum_{r=0}^{m-1}c_{\gamma,r}\sigma\Big(Tx-t-\frac{r}{m}\Big)=\gamma(Tx-t).
$$
\end{enumerate}
\end{proof}
Using definition \eqref{eq:f2i} of $\widetilde f_2$ and representation \eqref{eq:ggamma2}, we can then write, for any $x\in[0,1),$ 
\begin{align}
\widetilde f_2(x)
&=\frac{1}{T}\sum_{t=0}^{T-1}\theta_{\gamma_t}(\phi(Tx-t-1),\phi(Tx-t))\nonumber\\
&=\frac{1}{T}\sum_{\gamma\in\Gamma}\sum_{t:\gamma_t=\gamma}\theta_{\gamma}(\phi(Tx-t-1),\phi(Tx-t)).\nonumber
\end{align}
In order to obtain computational gain from caching, we would like to move summation  over $t$ into the arguments of the function $\theta_\gamma$. However, we need to avoid double counting associated with overlapping supports of the functions $x\mapsto \phi(Tx-t),t=0,\ldots,T-1$. Therefore we first divide all $t$'s into three series $\{t: t\equiv i \Mod{3}\}$ indexed by $i\in\{0,1,2\}$, and we will then move summation over $t$ into the argument of $\theta_\gamma$ separately for each series. Precisely, we write
\begin{equation}\label{eq:f2final0}
\widetilde f_2(x)=\frac{1}{T}\sum_{\gamma\in\Gamma}\sum_{i=0}^2\widetilde f_{2,\gamma,i},
\end{equation}
where
\begin{equation}\label{eq:f2final1}
\widetilde f_{2,\gamma,i}(x)=\sum_{\substack{t:\gamma_t=\gamma\\t\equiv i \Mod{3}}}\theta_{\gamma}(\phi(Tx-t-1),\phi(Tx-t)).
\end{equation}
We claim now that $\widetilde f_{2,\gamma,i}$ can be alternatively written as
\begin{equation}\label{eq:f2final}
\widetilde f_{2,\gamma,i}(x)=\theta_\gamma\Big(\sum_{\substack{t:\gamma_t=\gamma\\t\equiv i \Mod{3}}}\phi(Tx-t-1), \sum_{\substack{t:\gamma_t=\gamma\\t\equiv i \Mod{3}}}\phi(Tx-t)\Big).
\end{equation}
To check that, suppose that $x\in I_{t_0}$ for some $t_0$ and consider several cases: 
\begin{enumerate}
\item Let $i\not\equiv t_0 \Mod{3}$, then both expressions \eqref{eq:f2final1} and \eqref{eq:f2final} vanish. Indeed, at least one of the sums forming the arguments of $\theta_\gamma$ in the r.h.s. of \eqref{eq:f2final} vanishes by \eqref{eq:phivanish}, so $\theta_\gamma$ vanishes by \eqref{eq:ggamma11}, \eqref{eq:ggamma12}. Similar reasoning shows that the r.h.s. of \eqref{eq:f2final1} vanishes too.
\item Let $i\equiv t_0 \Mod{3}$ and $\gamma\ne\gamma_{t_0}$. Then all terms in the sums over $t$ in \eqref{eq:f2final1} and \eqref{eq:f2final} vanish, so both \eqref{eq:f2final1} and \eqref{eq:f2final} vanish.
\item Let $i\equiv t_0 \Mod{3}$ and $\gamma=\gamma_{t_0}$. Then all terms in the sums over $t$ in \eqref{eq:f2final1} and \eqref{eq:f2final} vanish except for the term $t=t_0$, so both \eqref{eq:f2final1} and \eqref{eq:f2final} are equal to $\theta_\gamma(\phi(Tx-t_0-1), \phi(Tx-t_0))$. 
\end{enumerate}
The desired ReLU expansion of $\widetilde f_2$ is then given by \eqref{eq:f2final0} and \eqref{eq:f2final}, where $\phi$ and $\theta_\gamma$ are further expanded by \eqref{eq:phiexp}, \eqref{eq:ggammaexp}:
\begin{alignat}{2}\label{eq:f2final2}
\widetilde f_2(x)&=\frac{1}{T}\sum_{\gamma\in\Gamma}\sum_{i=0}^2\widetilde f_{2,\gamma,i} &&\nonumber\\
&=\frac{1}{T}\sum_{\gamma\in\Gamma}\sum_{i=0}^2\sum_{r=0}^{m-1}c_{\gamma,r}\sigma\Big(&\frac{m-r}{m}&\sum_{\substack{t:\gamma_t=\gamma\\t\equiv i \Mod{3}}}\sum_{q=-1}^1\alpha_q\sigma(Tx-t-q-1)\nonumber\\
&&-\frac{r}{m}&\sum_{\substack{t:\gamma_t=\gamma\\t\equiv i \Mod{3}}}\sum_{q=-1}^1\alpha_q\sigma(Tx-t-q)\Big).
\end{alignat}


\begin{figure}[ht]
\begin{center}
\includegraphics[clip,trim=75mm 23mm 86mm 25mm]{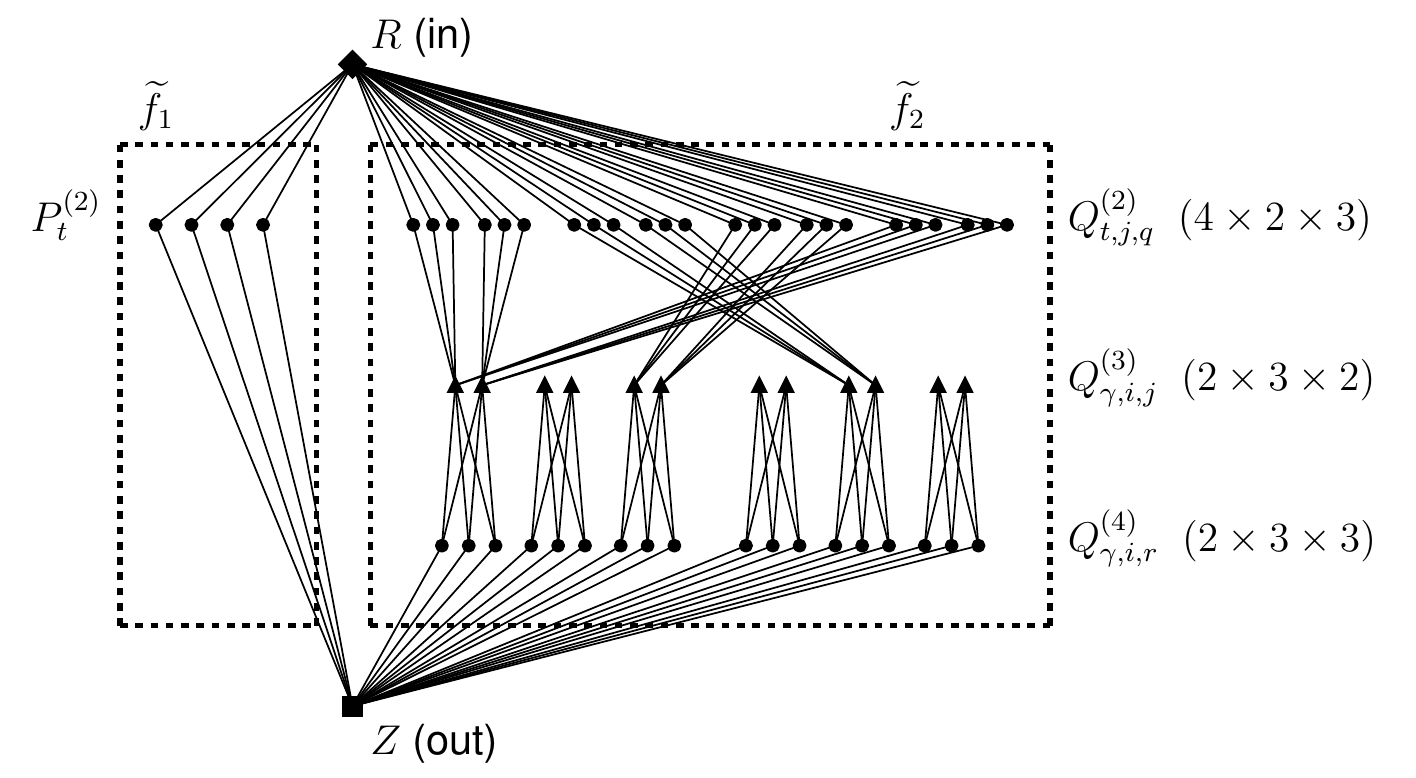}
\caption{Implementation of the function $\widetilde f=\widetilde f_1+\widetilde f_2$ by a neural network with $f$-dependent architecture. The computation units $Q^{(2)}_{t,j,q}, Q^{(3)}_{\gamma,i,j}$ and $Q^{(4)}_{\gamma,i,r}$ are depicted in their layers in the lexicographic order (later indices change faster). Dimensions of the respective index arrays are indicated on the right (the example shown is for $T=4,$ $m=3$ and $|\Gamma|=2$). Computation of $\widetilde f_2$ includes $3|\Gamma|$ parallel computations of functions $\widetilde f_{2,\gamma,i}$; the corresponding subnetworks are formed by units and connections connected to $Q^{(3)}_{\gamma,i,0}$ and $Q^{(3)}_{\gamma,i,1}$ with specific $\gamma,i$.}\label{fig:cache2}
\end{center}
\end{figure}

\paragraph{Step 3: Network implementation with $f$-dependent architecture.}
We will now express the approximation $\widetilde f$ by a neural network (see Fig. \ref{fig:cache2}). The network consists of two parallel subnetworks implementing $\widetilde f_1$ and $\widetilde f_2$ that include three and five layers, respectively (one and three layers if counting only hidden layers). The units of the network either have the ReLU activation function or simply perform linear combinations of inputs without any activation function. We will denote individual units in the subnetworks $\widetilde f_1$ and $\widetilde f_2$ by symbols $P$ and $Q$, respectively, with superscripts numbering the layer and subscripts identifying the unit within the layer. The subnetworks have a common input unit:
$$R=P^{(1)}=Q^{(1)}$$
and their output units are merged so as to sum $\widetilde f_1$ and $\widetilde f_2$:
$$Z=P^{(3)}+Q^{(5)}.$$

Let us first describe the network for $\widetilde f_1$. By \eqref{eq:f1exp}, we can represent  $\widetilde f_1$ by a 3-layer ReLU network as follows:
\begin{enumerate}
\item The first layer contains the single input unit $P^{(1)}$.
\item The second layer contains $T$ units $P^{(2)}_{t}= \sigma(P^{(1)}-\frac{t}{T})$, where $t\in\{0,\ldots,T-1\}$.
\item The third layer contains a single output unit $P^{(3)}=\sum_{t=0}^{T-1}w_tP^{(2)}+h$.
\end{enumerate}
Now we describe the network for $\widetilde f_2$ based on the representation \eqref{eq:f2final2}.
\begin{enumerate}
\item The first layer contains the single input unit $Q^{(1)}$.
\item The second layer contains $6T$ units $Q^{(2)}_{t,j,q}$, where $t\in\{0,\ldots,T-1\}$, $q\in\{-1,0,1\}$, and $j\in\{0,1\}$ corresponds to the first or second argument of the function $\theta_\gamma$:
$$Q^{(2)}_{t,j,q}=\sigma(TQ^{(1)}-t-q-j).$$  
\item The third layer contains $6|\Gamma|$ units $Q^{(3)}_{\gamma,i,j}$, where $\gamma\in\Gamma$, $i\in\{0,1,2\}$, and $j\in\{0,1\}$: 
\begin{equation}\label{eq:q3}Q^{(3)}_{\gamma,i,j}
=\sum_{\substack{t:\gamma_t=\gamma\\t\equiv i \Mod{3}}}\sum_{q=-1}^1\alpha_{q}Q^{(2)}_{t,j,q}.\end{equation}
\item The fourth layer contains $3m|\Gamma|$ units $Q^{(4)}_{\gamma,i,r}$, where $\gamma\in\Gamma, i\in\{0,1,2\}$ and $r\in\{0,\ldots,m-1\}$:
$$Q^{(4)}_{\gamma,i,r}=\sigma\Big(\frac{m-r}{m}Q^{(3)}_{\gamma,i,1}-\frac{r}{m}Q^{(3)}_{\gamma,i,0}\Big).$$
\item The final layer consists of the single output unit \begin{equation}\label{eq:q5}Q^{(5)}=\frac{1}{T}\sum_{\gamma\in\Gamma}\sum_{i=0}^2\sum_{r=0}^{m-1}c_{\gamma,r}Q^{(4)}_{\gamma,i,r}.\end{equation}
\end{enumerate}


\paragraph{Step 3: Embedding into the standard deep architecture.}
We show now that the $f$-dependent ReLU network constructed in the previous step can be realized within the standard width-5 architecture.

Note first that we may ignore unneeded connections in the standard network (simply by assigning weight 0 to them). Also, we may assume some units to act purely linearly on their inputs, i.e., ignore the  nonlinear ReLU activation. Indeed, for any bounded set $D\subset \mathbb R$, if $h$ is sufficiently large, then $\sigma(x+h)-h=x$ for all $x\in D$, hence we can ensure that the ReLU activation function always works in the identity regime in a given unit by adjusting the intercept term in this unit and in the units accepting its output. In particular, we can also implement in this way identity units, i.e. those having a single input and passing the signal further without any changes.

\begin{figure}
\begin{center}
\begin{subfigure}[b]{1\textwidth}
\begin{center}
\includegraphics[clip,trim=22mm 3mm 20mm 0mm]{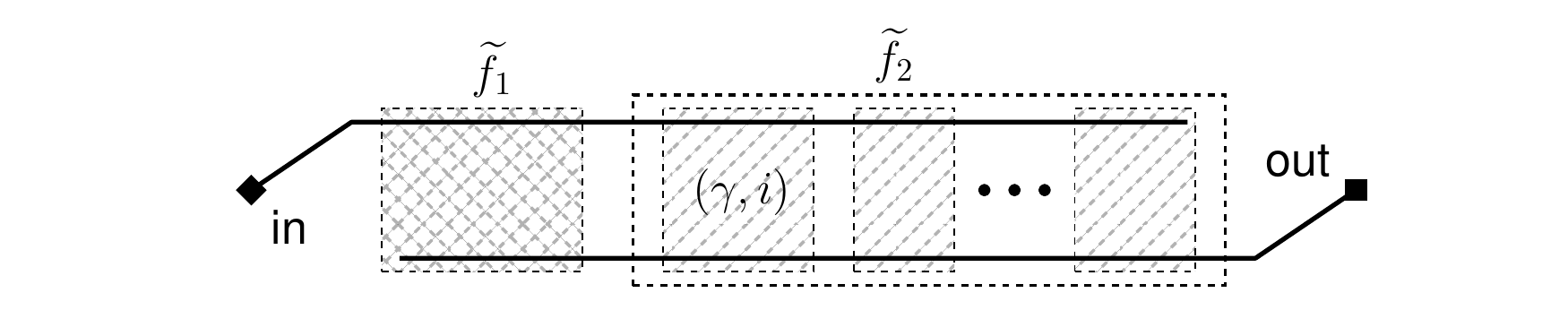}
\caption{Embedding overview: the subnetworks $\widetilde f_1$ and $\widetilde f_2$ are computed in parallel; the subnetwork $\widetilde f_2$ is further divided into parallel subnetworks $\widetilde f_{2,\gamma,i}$, where $\gamma\in \Gamma, i\in\{0,1,2\}$.}\label{fig:embedOverview}
\end{center}
\end{subfigure}
\begin{subfigure}[b]{1\textwidth}
\begin{center}
\includegraphics[clip,trim=15mm 2mm 15mm 0mm]{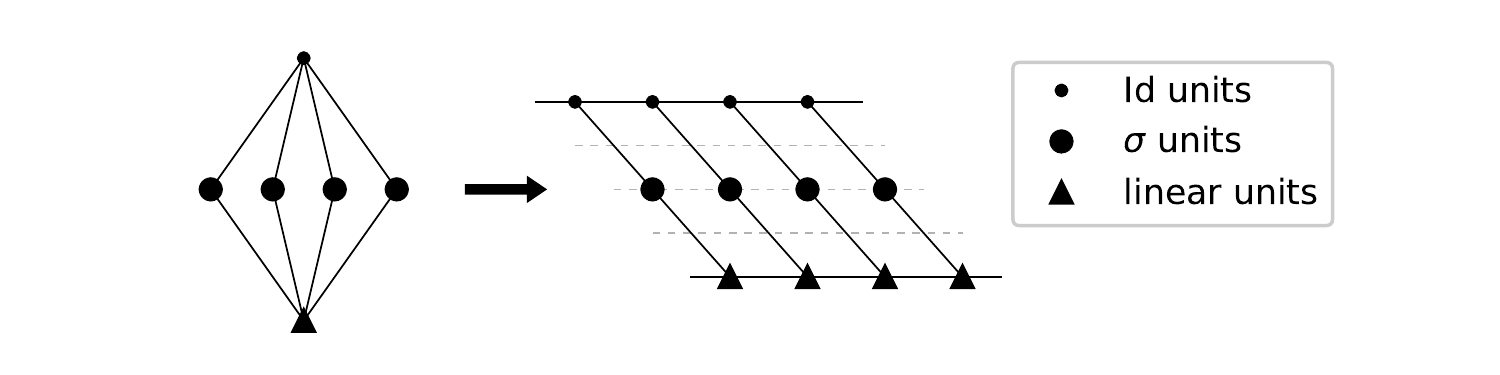}
\caption{Embedding of the $\widetilde f_1$ subnetwork.}\label{fig:embedf1}
\end{center}
\end{subfigure}
\begin{subfigure}[b]{1\textwidth}
\begin{center}
\includegraphics[clip,trim=25mm 7mm 22mm 3mm]{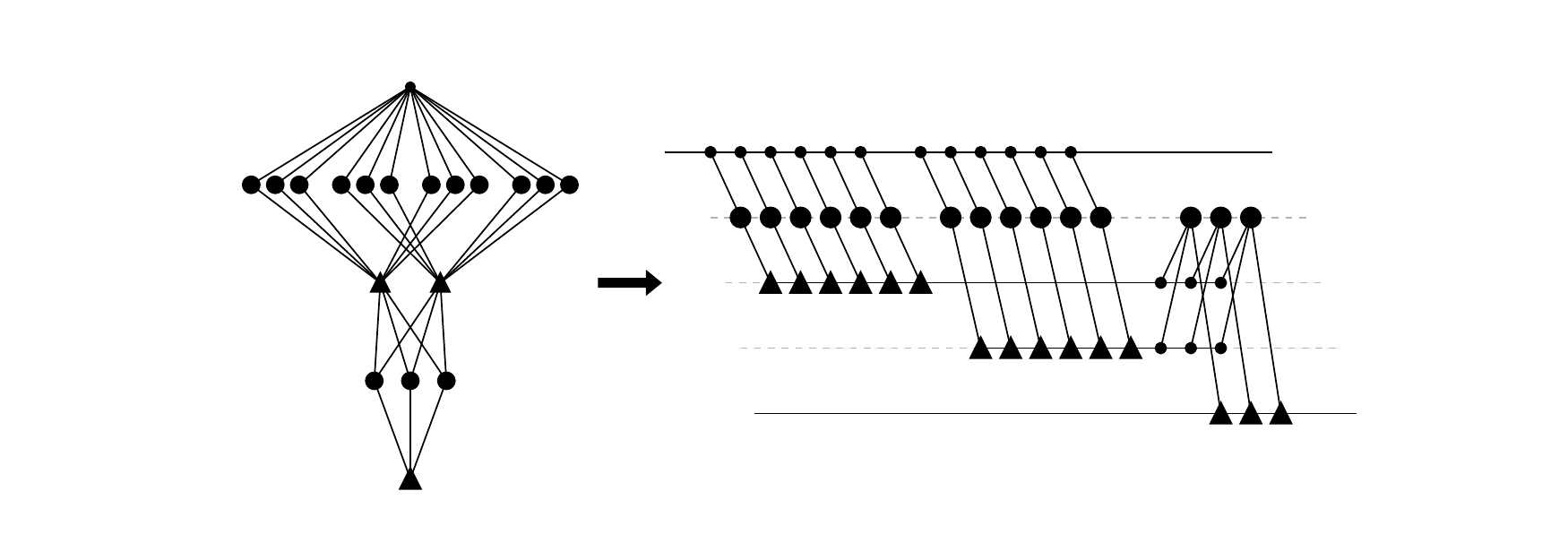}
\caption{Embedding of a subnetwork $\widetilde f_{2,\gamma,i}$.}\label{fig:ebedf2}
\end{center}
\end{subfigure}
\caption{Embedding of the adaptive network into the standard architecture.}
\end{center}
\end{figure}

The embedding strategy is to arrange parallel subnetworks along the  ``depth'' of the standard architecture as shown in Fig. \ref{fig:embedOverview}. Note that  $\widetilde f_1$ and $\widetilde f_2$ are computed in parallel; moreover, computation of $\widetilde f_2$ is parallelized over $3|\Gamma|$ independent subnetworks computing $\widetilde f_{2,\gamma,i}$ and indexed by $\gamma\in \Gamma$ and $i\in\{0,1,2\}$ (see \eqref{eq:f2final2} and Fig. \ref{fig:cache2}). Each of these independent subnetworks gets embedded into its own batch of width-5 layers. The top row of units in the standard architecture is only used to pass the network input to each of the subnetworks, so all the top row units function in the identity mode. The bottom row is only used to accumulate results of the subnetworks into a single linear combination, so all the bottom units function in the linear mode.   

Implementation of the $\widetilde f_1$ subnetwork is shown in Fig. \ref{fig:embedf1}. It requires $T+2$ width-5 layers of the standard architecture. The original output unit $P^{(3)}$ of this subnetwork gets implemented by $T$ linear units that have not more than two inputs each and gradually accumulate the required linear combination. 

Implementation of a  $\widetilde f_{2,\gamma,i}$ subnetwork is shown in Fig. \ref{fig:ebedf2}. Its computation can be divided into two stages. 

In terms of the original adaptive network, in the first stage we perform parallel computations associated with the $Q^{(2)}$ units and combine their results in the two linear units $Q^{(3)}_{\gamma,i,0}$ and $Q^{(3)}_{\gamma,i,1}$. By \eqref{eq:q3}, each of $Q^{(3)}_{\gamma,i,0}$ and $Q^{(3)}_{\gamma,i,1}$ accepts $3N_{\gamma,i}$ inputs from the $Q^{(2)}$ units, where $$N_{\gamma,i}=|\{t\in\{0,\ldots,T-1\}|\gamma_t=\gamma,t\equiv i\Mod{3}\}|.$$
In the standard architecture, the two original linear units $Q^{(3)}_{\gamma,i,0}$ and $Q^{(3)}_{\gamma,i,1}$ get implemented by $6N_{\gamma,i}$ linear units that  occupy two reserved lines of the network. This stage spans $6N_{\gamma,i}+2$ width-5 layers of the standard architecture. 

In the second stage we use the outputs of $Q^{(3)}_{\gamma,i,0}$ and $Q^{(3)}_{\gamma,i,1}$ to compute $m$ values $Q^{(4)}_{\gamma,i,r}, r=0,\ldots,m-1$, and accumulate the respective part $\sum_{r=0}^{m-1}c_{\gamma,r}Q^{(4)}_{\gamma,i,r}$ of the final output \eqref{eq:q5}. This stage spans $m+2$ layers of the standard architecture.

The full implementation of one $\widetilde f_{2,\gamma,i}$ subnetwork thus spans $6N_{\gamma,i}+m+4$ width-5 layers. Since $\sum_{\gamma,i}N_{\gamma,i}=T$ and there are $3|\Gamma|$ such subnetworks, implementation of the whole $\widetilde f_2$ subnetwork spans $6T+3(m+4)|\Gamma|$ width-5 layers. Implementation of the whole $\widetilde f$  network then spans $7T+3(m+4)|\Gamma|+2$ layers. 

It remains to optimize $T$ and $m$ so as to achieve the minimum approximation error, subject to the total number of layers in the standard network being bounded by $N$. Recall that $|\Gamma|\le 3^m$ and that the approximation error of $\widetilde f$ is not greater than $\frac{2}{Tm}$ by \eqref{eq:f2tm}. Choosing $T=\lfloor\frac{N}{8}\rfloor$, $m=\lfloor\frac{1}{2}\log_3 N\rfloor$ and assuming $N$ sufficiently large, we satisfy the network size constraint and achieve the desired error bound $\|f-\widetilde f\|_\infty\le\frac{c}{N\ln N}$, uniformly in $f\in B$. Since, by construction, $\widetilde f=\eta_{5,N}(\psi(f))$ with some weight selection function $\psi$, this completes the proof.


\bibliographystyle{plain}
\bibliography{relu}

\end{document}